\documentclass[12pt]{ieeeconf}
\usepackage{calc}
\usepackage{url}
\usepackage{hyperref}
\hypersetup{
  colorlinks =true,
  urlcolor = black,
  linkcolor = black
}

\usepackage{graphicx}
\usepackage[cmex10]{amsmath}
\usepackage{amssymb}
\usepackage{rotating}
\usepackage{balance}

\usepackage{nicefrac}
\usepackage{cite}
\usepackage[caption=false,font=footnotesize]{subfig}
\usepackage[usenames, dvipsnames]{color}
\usepackage{colortbl}
\usepackage{overpic}
\graphicspath{{./pictures/},{./pictures/pdf/},{./pictures/ps/},{./pictures/png/},{./pictures/jpg/}}
\usepackage{breqn} 
\usepackage[ruled]{algorithm}
\usepackage{algpseudocode}
\usepackage{multirow}

\usepackage{bm}   
\usepackage{soul}  
\usepackage{bbm}

\newtheorem{theorem}{Theorem}

\newtheorem{corollary}[theorem]{Corollary}

\begin{document}



\title{\LARGE \bf 
Collecting a Swarm in a Grid Environment Using Shared, Global Inputs
}
\author{Arun V. Mahadev,
Dominik Krupke,
Jan-Marc Reinhardt,
S\'andor P. Fekete,
  Aaron T.\ Becker
\thanks{{A.~Mahadev and A.~Becker are with the Department of Electrical and Computer Engineering,  University of Houston, Houston, TX 77204-4005 USA 
      \protect\url{ aviswanathanmahadev@uh.edu,atbecker@uh.edu }
S.~Fekete, D.~Krupke, and
J.M.~Reinhardt are with the Dept.~of Computer Science, TU Braunschweig,  M\"uhlenpfordtstr.~23, 38106 Braunschweig, Germany,
      \protect\url{s.fekete@tu-bs.de,j-m.reinhardt@tu-bs.de,d.krupke@tu-bs.de  }
}
} 
} 
\maketitle

\begin{abstract}
This paper investigates efficient techniques to collect and concentrate an under-actuated particle swarm despite obstacles. 
Concentrating a swarm of particles is of critical importance in health-care for targeted drug delivery, where micro-scale particles must be steered to a goal location.
Individual particles must be small in order to navigate through micro-vasculature, but decreasing size brings new challenges. 
Individual particles are too small to contain on-board power or computation and are instead controlled by a global input, such as an applied fluidic flow or electric field. 

To make progress, this paper considers a swarm of robots initialized in a grid world in which each position is either free-space or obstacle.
This paper provides algorithms that collect all the robots to one position and compares these algorithms on the basis of efficiency and implementation time.
\end{abstract}

  \section{Introduction}
  Targeted drug therapy is a goal for many interventions, including treating cancers, delivering pain-killers, and stopping internal bleeding. Treatment often uses the patient's vasculature to deliver the therapy. This drug therapy is challenging due to the complicated geometry of vasculature, as shown in Fig.~\ref{fig:vascularNetwork}.

  This paper builds on the techniques for controlling many simple robots  with uniform control inputs presented in \cite{Becker2014,Becker2014a}, and also outlines new research problems; see video and abstract~\cite{bmd+-pcdfbm-15} for a visualizing overview.
  
    \begin{figure}
\centering
\begin{overpic}[width=0.9\columnwidth]{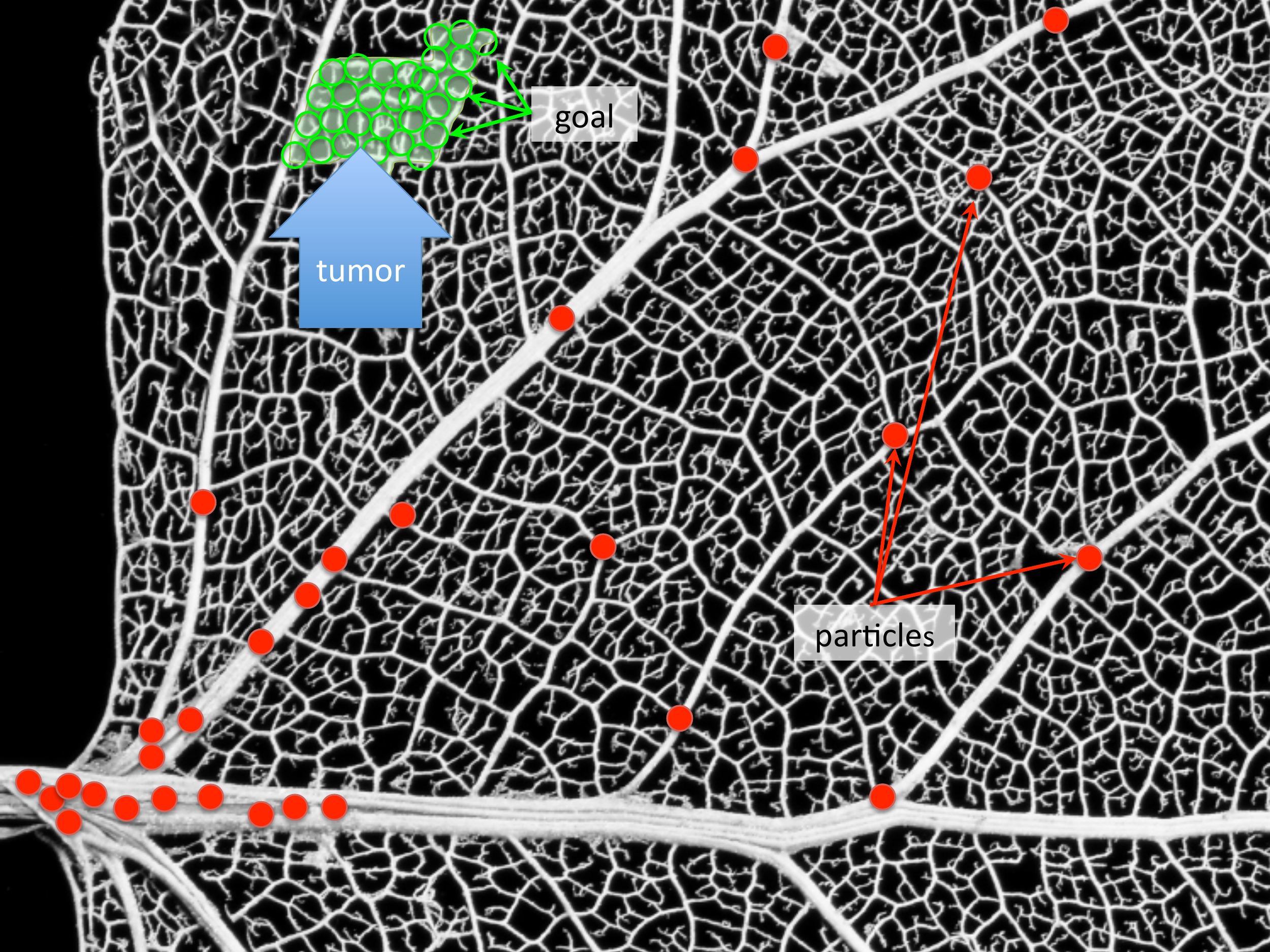}\end{overpic}
\caption{\label{fig:vascularNetwork}
Vascular networks are common in biology such as the circulatory system and cerebrospinal spaces, as well as in porous media including sponges and pumice stone.  Navigating a swarm using global inputs, where each member receives the same control inputs, is challenging
due to the many obstacles. 
This paper focuses on using boundary walls to break the symmetry and collect the swarm at a desired location. See simulation at \cite{youtube}.} 
\end{figure}
\cite{Chowdhury2015} gives us an understanding of some of the  challenges related to controlling multiple micro robots (less than 64 robots at a time).  
Building systems capable of accomplishing difficult motion tasks is a major focus of research in this area and \cite{cheang2014multiple} shows how magnetic manipulation has great potential controlling such particles in low Reynolds number. One example is particles with a magnetic core and a catalytic surface for carrying medicinal payloads~\cite{pouponneau2009magnetic,litvinov2012high}.
  An alternative is \emph{superparamagnetic iron oxide microparticles},  9 $\mu$m particles that are used as a contrast agent in MRI studies~\cite{mellal2015magnetic}. Real-time MRI scanning could allow feedback control using the location of a swarm of these particles.


  Steering magnetic particles using the magnetic gradient coils in an MRI scanner was implemented in \cite{mathieu2007magnetic, pouponneau2009magnetic}.  
 3D Maxwell-Helmholtz coils are often used for precise magnetic field control \cite{mellal2015magnetic}. Still needed are motion planning algorithms to guide the swarms of robots through vascular networks.

 From an algorithmic perspective, the strongest parallels in literature are robot localization and rendezvous. In
 \emph{Almost
sensorless localization} or ``localizing a blind robot in a known map",
a mobile robot with a map of the workspace must localize itself,
using only a compass and a bump-sensor that detects when the robot
contacts a wall.
~\cite{o2005almostPhD,o2005almost} has been extended to robots with bounded uncertainty in their inputs~\cite{Lewis01092013}. Given an environment, finding a localizing sequence is framed as a planning problem with an unknown initial state and an unobservable current state.  
  The solution in \cite{o2005almostPhD} was to transform the problem from an unobservable planning problem in state space to an observable problem in a more complex  information space.  They provided a complete
algorithm, but  generating an optimal localizing sequences remains an open problem.
~\cite{o2005almostPhD} assumes there is only one robot, but it still gives us clarity on how planning for independent robot systems differ from swarm robot systems. 
Also related is work on sensorless part orientation, where a flat tray is tilted in a series of directions to bring a polygonal part, initially placed at random orientation and position in the tray, to a known position and orientation~\cite{Akella2000a}.
This  is similar to localizing a robot with minimum travel; however, it moves the robot and requires take additional measurements\cite{dudek1998localizing}.  

 The other parallel concept, \emph{robot rendezvous}, requires two
or more independent, intelligent agents to meet. 
Alpern and Gal~\cite{alpernbook}
introduced a wide range of models and methods for this concept as have Anderson and Fekete~\cite{anderson2001two} in a two-dimensional geometric setting. Key assumptions include a bounded topological environment and robots with limited onboard computation.
This is relevant to maneuvering particles through worlds with obstacles and implementation of strategies to reduce computational burden while calculating distances in complex worlds \cite{meghjani2012multi}. 
In a setting with autonomous robots, these can move independent of each other, i.e., follow
different movement protocols, 
called {\em asymmetric} rendezvous in the mathematical literature~\cite{alpernbook}.
If the agents are required to follow the same protocol, this is called {\em symmetric} rendezvous.
This corresponds to our model in which particles are bound by the
uniform motion constraint; symmetry is broken only by interaction with the obstacles.

The `robots' in this paper are simple particles without autonomy.
A planar  grid \emph{workspace} $W$ is filled with a number of unit-square robots (each occupying one cell of the grid)  and some fixed unit-square blocks.  Each unit square in the workspace is either  \emph{free}, which a robot may occupy or \emph{obstacle} which a robot may not occupy.  Each square in the grid can be referenced by its Cartesian coordinates $\bm{x}=(x,y)$.
All robots are commanded in unison: the valid commands are  ``Go Up" ($u$), ``Go Right" ($r$), ``Go Down" ($d$), or ``Go Left" ($l$).

We consider two classes of commands, discrete and maximal moves.
\emph{Discrete moves}: robots all move in the commanded direction one unit unless they are prevented from moving by an obstacle or a stationary robot. \emph{Maximal moves}: robots all move in the commanded direction until they hit an obstacle or a stationary robot. For maximal moves, we assume the area of $W$ is finite and issue each command long enough for the robots to reach their maximum extent.
A command sequence $\bm{m}$ consists of an ordered sequence of moves $m_k$, where each $m_k\in\{u,d,r,l\}$  
A representative command sequence is $\langle u,r,d,l,d,r,u,\ldots\rangle$.

  We consider two types of particles, small and large, as depicted
in Fig. \ref{fig:SmallVsLarge}. If particles are much smaller than the
workspace geometry, we call them small. We represent each grid cell as
filled if it contains at least one particle and empty otherwise. A cell filled
with small particles can combine with another filled cell.  If  particles are
the same size as workspace gridcells, we call the particles large.
Large particles cannot combine.  The presence of a  large particle in a cell
prevents another particle from entering.
  
  We study two notions of collecting a swarm, corresponding with particle size:
for small particles the swarm is collected when all robots share the same
$(x,y)$ coordinates. If the particles are large, the swarm is collected when
it forms one connected component.  
2D cells are neighbors if they share an edge, 3D cells are neighbors if they share a face.  
A \emph{connected component} is a set of particles $P$ such that for any two particles in $P$, there is a sequence of neighboring particles that connect them.

\begin{figure}
\centering
\begin{overpic}[width=1.0\columnwidth]{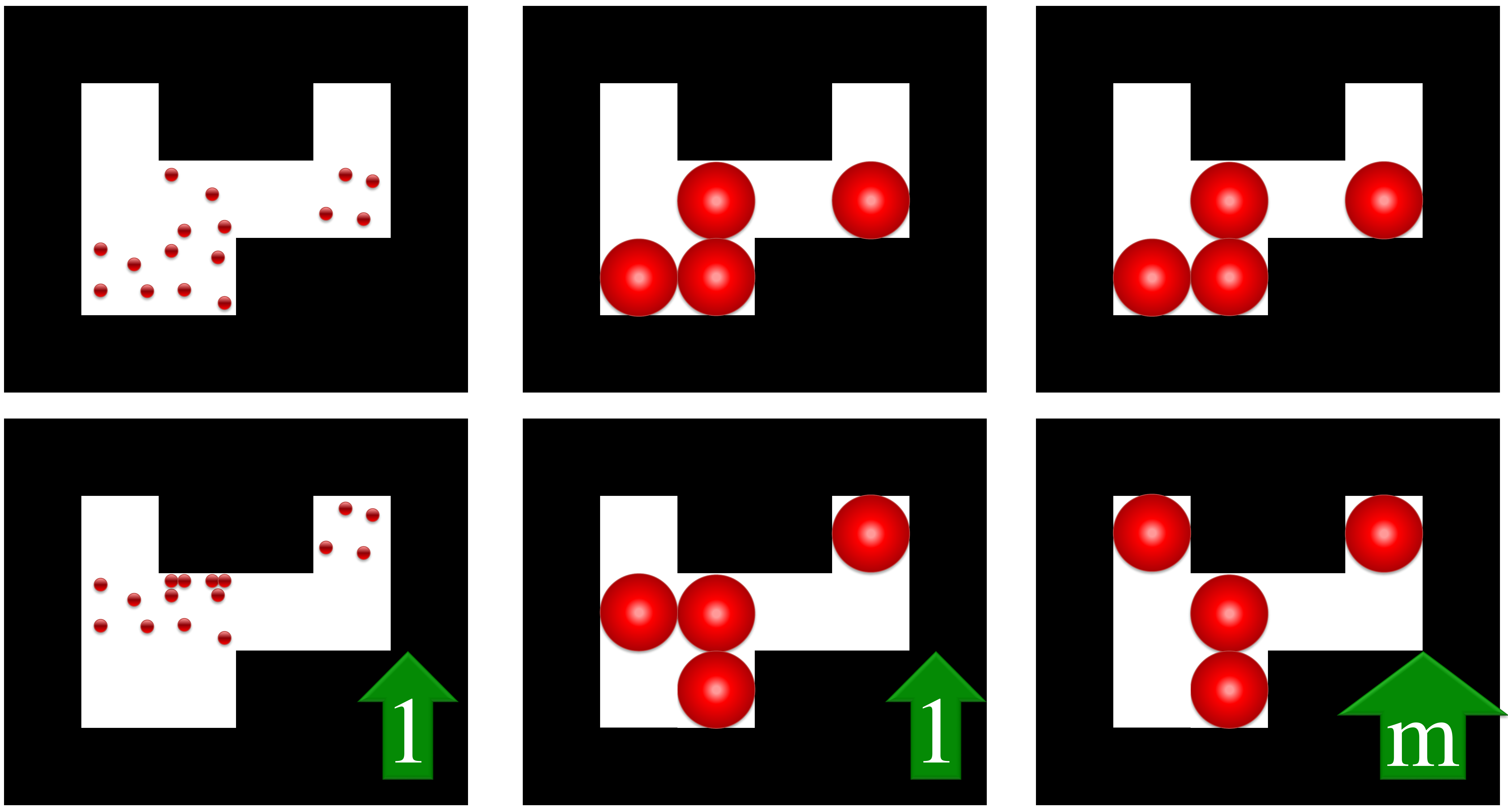}\end{overpic}
\caption{\label{fig:SmallVsLarge}
If particles are much smaller than the workspace geometry, we call them {\em small}.
We represent each grid cell as filled or empty, and allow a filled cell
to combine with another filled cell.  If  particles are the same size as
workspace gridcells, we call the particles {\em large}.  Large particles cannot
combine.  The presence of a  large particle in a cell prevents another particle
from entering.} 
\end{figure}

  \section{Theory}\label{sec:theory}

 With discrete inputs and small particles, the problem can be reduced to
localizing a sensorless robot in a known workspace. This is similar to work
on draining a polygon~\cite{aloupis2014draining}, or localizing a blind
robot~\cite{o2005almostPhD,o2005almost}, but with discrete inputs. In Section~\ref{sec:Results}, for the
small particle problem we present an optimal collection algorithm, Alg.~1 in
Section~\ref{Alg1} and a greedy collection policy Alg.~2 in
Section~\ref{Alg2}. We also give positive (Section~\ref{Alg3}) and negative
(Section~\ref{Alg4}) results for large particles.
\subsection{Our problems of interest}
\label{sec:problem of interest}
  
  The freespace must be connected. Robots initialized in two unconnected components $i$ and $j$ of a free space cannot be collected. The proof is trivial, since a robot in free space $i$ can not reach free space $j$. Such a configuration is depicted in Fig.~\ref{fig:impossible configurations}$a$.
  
  Under maximal inputs, the world can be constructed with spaces resembling bottles or fish weirs from which a single robot cannot escape, as shown in Fig.~\ref{fig:impossible configurations}$b$.
  If the free space contains at least two such bottles with at least one robot in each, the swarm cannot be collected with maximal inputs.
  
The world must be bounded. Two initially separated robots in an unbounded world without obstacles cannot be collected; however with discrete inputs, one obstacle is sufficient as seen in Fig.~\ref{fig:impossible configurations}$c$ and can be inferred from~\cite{Becker2013b}.
   \begin{figure}
  \begin{overpic}[width=\columnwidth]{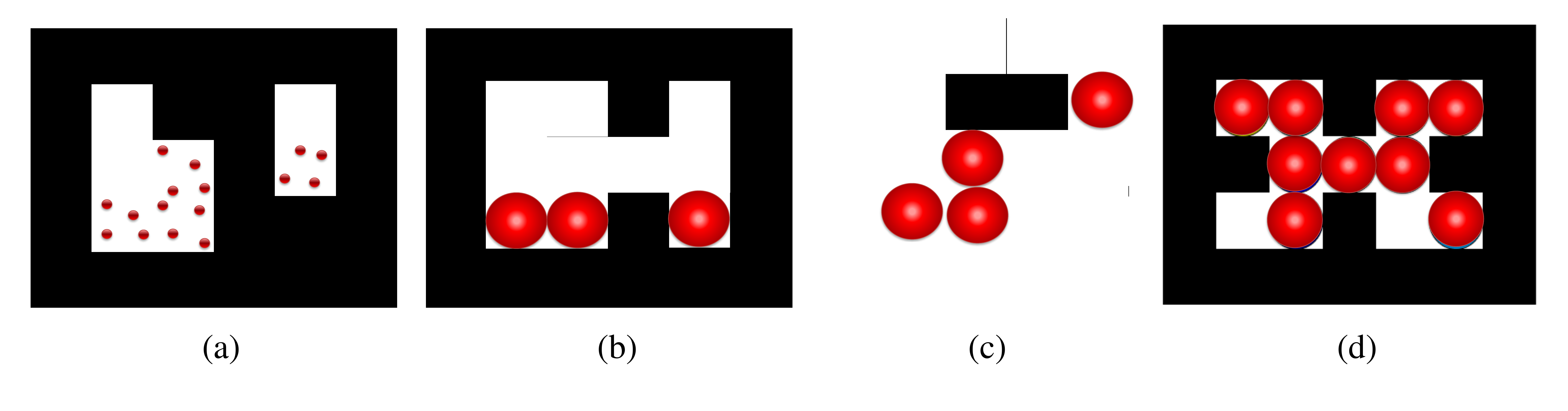}\end{overpic}
  \caption{\label{fig:impossible configurations}  Examples of workspaces for which collection is not possible.
 (a) The swarms are in unconnected components  
 (b) A world for which $maximal$ moves will never allow particles to meet. 
 (c) An unbounded world with a single obstacle.  In this world $discrete$ moves can collect particles but $maximal$ moves cannot collect particles.
 In a world without boundaries and without obstacles, discrete moves cannot collect all particles.
 (d) A world configuration where $large$ particle collection is impossible. No input sequence exists that will make all the particles part of the same connected component. 
 }
  \end{figure}

\emph A swarm with discrete moves and small particles can be collected on any bounded grid.
However, with large particles there are configurations where the topology does not allow collection, as seen in Fig. \ref{fig:impossible configurations}$d$.      

  \subsection{Collecting with the shortest move sequence}
  \label{Alg1}
  A conceptually simple strategy to collect all particles in a workspace is to
construct a configuration tree that expands the tree of all possible movement
sequences in a breadth-first search manner, and halts when the configuration
has all robots collected at one point. Alg.~\ref{alg:OptimalCollectWithOverlap} implements this breadth-first-search
technique. 
  It initializes a tree where each node contains the configuration of robot locations $C[p]$ , the move that generated this configuration $M[p]$ and a parent configuration pointer $P[p]$. Here $C,M,P$ are the respective complete lists.  $p$ is the current iteration pointer and $e$ is the end of list pointer. 
  The root node is $\{ C_0, \varnothing, 0 \}$, where $C_0$ is the initial configuration of robot locations.
We then construct a breadth-first tree of possible configurations $\{ u,r,d,l\}$, pruning configurations that already exist in the tree. We stop when the cardinality at a leaf is one, $|C_i| = 1$, which indicates that the swarm has been collected (equivalently, that the robot has been localized).
This algorithm produces the optimal path to determine the shortest path length `s' as seen in Fig.~\ref{fig:CollectionSolution1}, but requires $O(4^{\text{s}})$ time to learn and $O(4^{\text{s}})$ memory and the graph grows exponentially (Fig.~\ref{fig:ConvergenceTimesForOptimal}). This leads us to investigate other algorithms which will solve the path with much lower computational time and data.

 \begin{figure*}[htb!]
  \begin{overpic}[width=\linewidth]{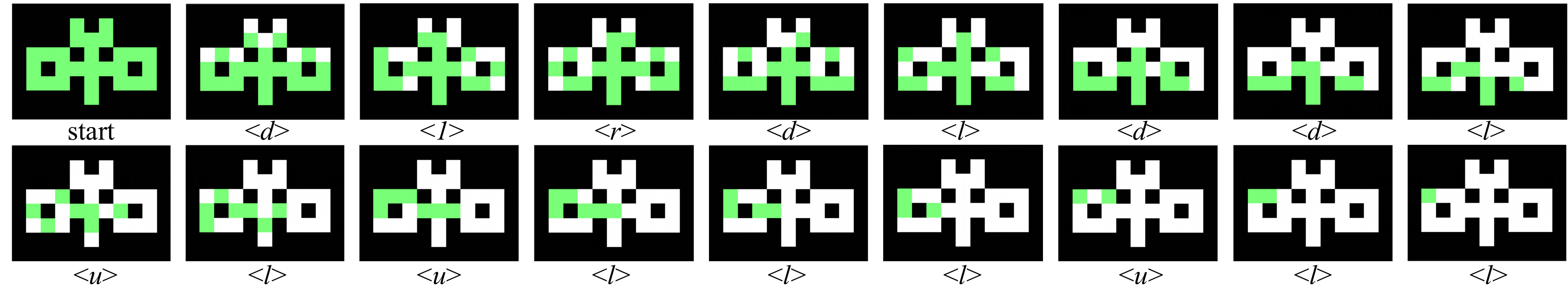}\end{overpic}
  \caption{\label{fig:CollectionSolution1}  With discrete inputs and particles, the collecting problem can be reduced to localizing a sensor-less robot in a known workspace. Above shows the optimal solution for a world with 27 free spaces, which required expanding 423,440 nodes with an optimal path (shown) taking 17 moves. 
 }
  \end{figure*}

  \begin{figure}
  \begin{overpic}[width=\columnwidth]{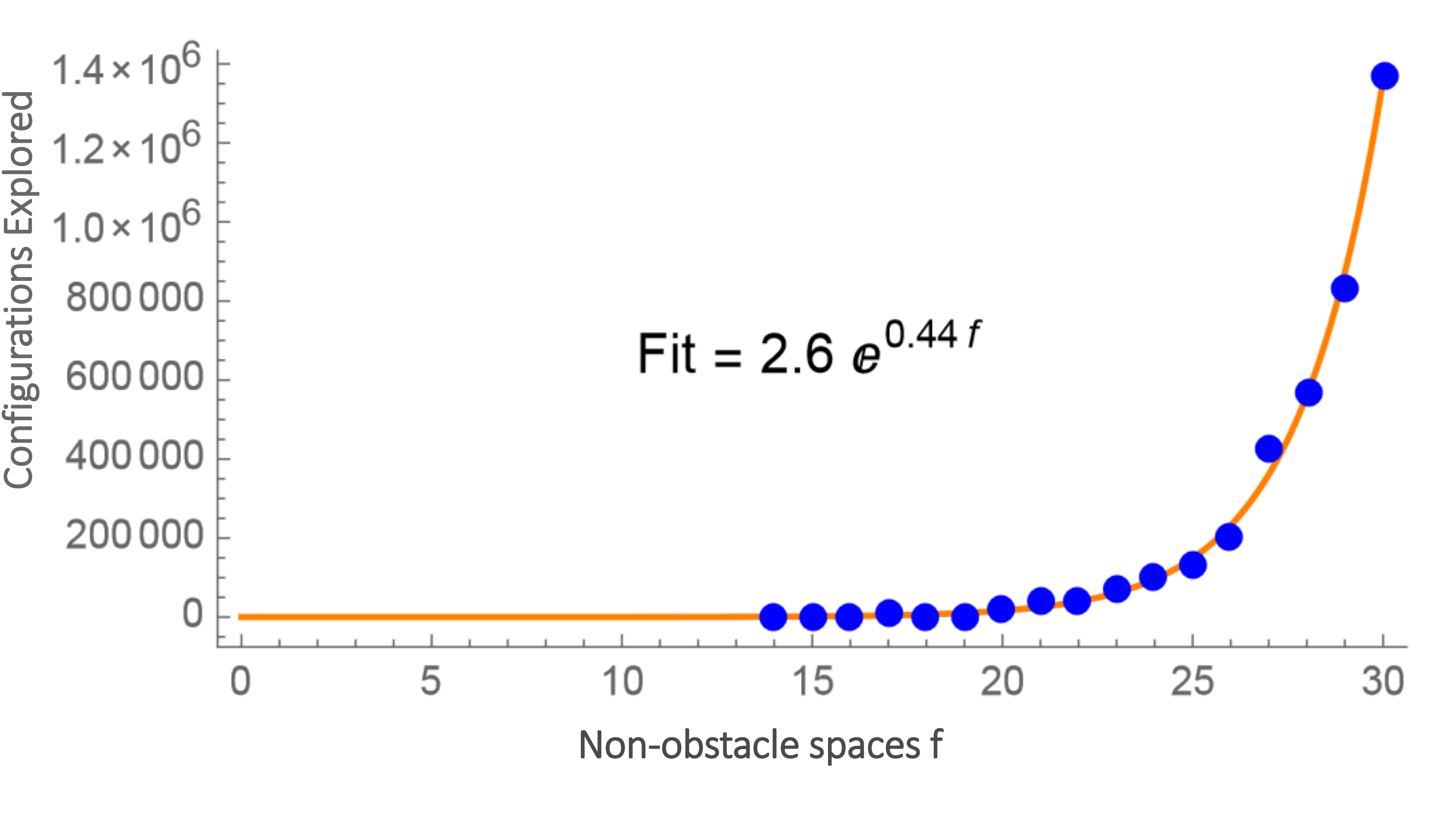}\end{overpic}
  \caption{\label{fig:ConvergenceTimesForOptimal} Unfortunately, the optimal BFS (Algorithm \ref{alg:OptimalCollectWithOverlap}) solution requires expanding a number of nodes that increases (approximately) exponentially with the number of free spaces.  A workspace with 30 free spaces required 1.6 million nodes before finding the optimal solution.
 }
  \end{figure}

  \begin{algorithm}
\caption{OptimalCollecting($W$, $C_0$)}\label{alg:OptimalCollectWithOverlap}
\begin{algorithmic}[1]
\State $p \gets  1$ 
\State $\{C[p],M[p],P[p]\} \gets \{ C_0,   \varnothing,  0\}$\Comment{initialize}
\State $e \gets  1$ 
\While{ $|C[p]|>1  $} \Comment{more than 1 unique position}
\For{$m = \{u,d,r,l\}$} 
\State  $C_{temp} \gets  $  ApplyMove($C[p], m$)
\If{$C_{temp} \not\in C $} \Comment{add node to list}
 \State  $e \gets  e+1$ 
\State $\{C[e],M[e],P[e]\} \gets \{  C_{temp},   m, p\}$
\EndIf
 \EndFor
  \State  $p \gets  p+1$ \Comment{get next configuration}
\EndWhile
\State $path \gets \{\}$ \Comment{construct optimal path}
\While{ $P[p]>1 $} 
\State Append $M[p]$ to $path$
\State $p \gets P[p]$
\EndWhile
\State $path \gets$ Reverse[$path$]
\end{algorithmic}
\end{algorithm}

  \subsection{Collecting small particles with a greedy strategy}
\label{Alg2}

Two particles in a finite and connected polyomino can be collected with small particles and discrete movement by simply repeatedly moving one particle onto another in the shortest way.
The corresponding procedure \textsc{CollectAB} is described in Alg.~\ref{alg:collectab}.
By iteratively collecting any two disjoint particles, the size of the distinct positions of the particle swarm can be reduced until all particles are at the same position.
The two particles can be chosen with different methods and our focus will be to implement the following methods:
 \begin{enumerate}
  \item Closest pair of particles - choose a pair of particles with the minimum distance between them.
\item Furthest pair of particles - choose a pair with maximum distance between them.
\item Connect to first - choose the first two  particles while searching for particles in the workspace from top left to bottom right.
\item Random combinations - choose any two particles. 
\item First to last - choose the first particle and last particle, i.e., the leftmost top and rightmost bottom particles, respectively. 
\end{enumerate}
\textsc{CollectAB} can be called to implement any of these methods. 

\begin{algorithm}
	\caption{Collecting two particles that can overlap}\label{alg:collectab}
	\begin{algorithmic}[1]
		\Require $a$ can reach $b$, Polyomino is bounded
		\Procedure{CollectAB}{a: Particle, b: Particle}
		\While{$dist(a,b)\not = 0$}
		\State Let $\mathcal{C}\in \{u,d,l,r\}^{N}$ be the shortest control sequence that moves $a$ onto $pos(b)$ 
		\State Execute $\mathcal{C}$
		\EndWhile
		\EndProcedure
	\end{algorithmic}
\end{algorithm}

\begin{theorem}
	\label{theorem:collectab:controlcommands}
	\textsc{CollectAB} collects two particles in a polyomino with $O(n^3)$ discrete control commands, where $n$ equals the polyomino's height times its width.
\end{theorem}
\begin{proof}
	The distance between $a$ and $b$ equals the length of $\mathcal{C}$.
	After execution of $\mathcal{C}$, the distance has not increased as $a$ is now on the previous position of $b$ and $b$ has at most moved $|\mathcal{C}|$ units from it.
	If $b$ had a collision during the execution of $\mathcal{C}$, the distance is even less as at least one command did not result in a move of $b$.
	As $dist(a,b)\in O(n)$, only $O(n)$ loop iterations with collisions are needed to collect $a$ and $b$.
	Obviously, $|\mathcal{C}|\in O(n)$ and hence every loop iteration executes at most $O(n)$ commands.
	With every iteration without collision, the positions of $a$ and $b$ change each by $pos(b)-pos(a)$.
	This difference only changes if $b$ had a collision, therefore the particles move in the same direction with every collision-free iteration.
	After $O(n)$ collision-free iterations of the loop, $b$ must have a collision, as the polyomino is finite.
	This results in $O(n^2)$ commands to reduce the distance by at least one and thus $O(n^3)$ commands suffice to collect $a$ and $b$.
\end{proof}

\begin{theorem}
	\textsc{CollectAB} has a computational complexity of $O(n^3)$.
\end{theorem}
\begin{proof} The shortest control sequence $\mathcal{C}$ can be calculated in $O(n)$ time by a simple breadth-first-search.
Under the assumption that a command can be executed in $O(1)$, one loop iteration has a computational complexity of $O(n)$. With $O(n^2)$ loop iterations (see proof of Theorem~\ref{theorem:collectab:controlcommands}), this results in an overall complexity of $O(n^3)$.
\end{proof}

\begin{theorem}
A particle swarm of size $O(m)$ can be collected with $O(m*n^3)$
discrete control commands and a computational complexity of $O(m*n^3)$ where
$n$ equals the polyomino's height times its width.  \end{theorem}
\begin{proof} Select two disjunct particles and execute $\mathtt{CollectAB}$.
	This reduces the size of distinct positions in the particle swarm by one.
	After $O(m)$ executions, there is only one position left and the particle swarm is collected.
\end{proof}

%
  

  \subsection{Collecting large particles in a target region}
\label{Alg3}

In the previous two subsections, the particles are relatively small,
allowing several to be collected in the same location 
anywhere in the environment. If the particles are relatively large, they may block each other's way,
making the motion control trickier.
We can still deliver a swarm of particles to a target region by making use of discrete moves, 
assuming that particles are metabolized once they reach the target region, i.e., the target is ``sticky''.
(This implies that they stay within the target region once they get there, and 
that they do not block each other within that region.)

\begin{theorem}
\label{th:large}
For a sticky target region and large particles within an
environment of diameter $D$, a particle swarm of
size $O(m)$ can be collected with $O(m \times D)$ discrete control commands.
\end{theorem}

\begin{proof} The proof is straightforward by induction. Moving one particle to the target region takes at most $D$ moves, 
which leaves all other particles within distance $D$.
\end{proof}

Note that the extent of the environment is critical. If we are dealing with $m$ particles within an environment of size $n\times n$,
then we get the following.

\begin{corollary}
\label{co:large+nxn}
For a sticky target region and large particles, a particle swarm within
an environment of size $n\times n$ can be collected with $O(n^3)$ discrete control commands.
\end{corollary}

%
%

If all particles of the swarm are relatively close to the target region, the complexity can be stated differently.

\begin{corollary}
\label{co:large+m}
For a sticky target region and large particles, a particle swarm of size $O(m)$ 
that fills a square environment around a target region can be collected with $O(m^{3/2})$ discrete control commands.
\end{corollary}

  \subsection{Collecting large particles with maximal moves}
\label{Alg4}

Our results rely on being able to limit the extent of the motion, i.e., having {\em discrete moves}.
If that is not the case, i.e., in the case of {\em maximal moves}, in which each particle moves
until it is stopped by an obstacle or another stopped particle, the problem becomes considerably harder
and may indeed be intractable.  
As we showed in previous work~\cite{Becker2013f}, deciding whether even a 
{\em single} particle can be delivered to a target region  (Fig.~\ref{fig:NPc}) is already an NP-hard problem; this implies
not only that finding a solution is computationally hard, but that there are instances in which
no solution exists.

\begin{figure}
\centering
\begin{overpic}[width=1.0\columnwidth]{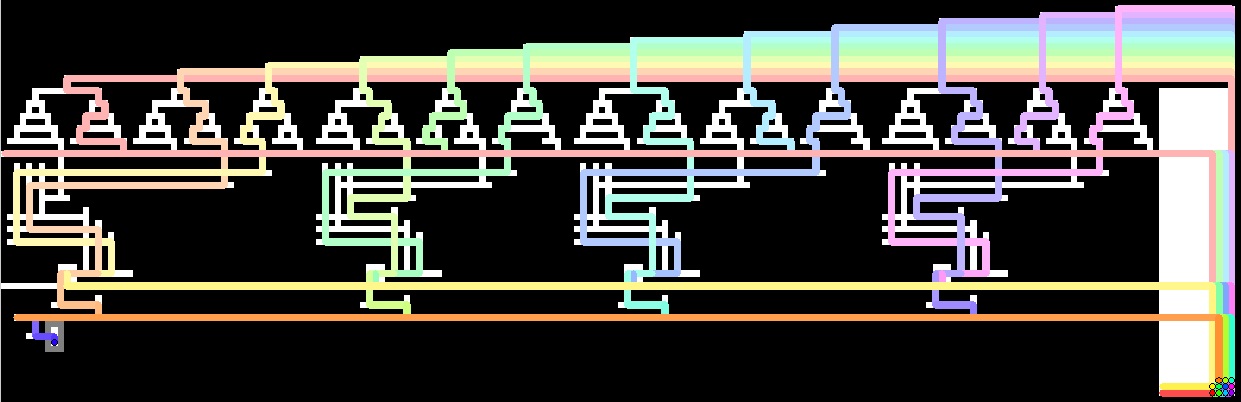}\end{overpic}
\caption{\label{fig:NPc}
NP-hardness of deciding reachability with maximal moves: In order to deliver a {\em single} ``blue'' particle into the grey target region in the lower left,
a 3SAT instance has to be satisfied, corresponding to a set of decisions in the upper part of the construction.}
\end{figure}

%


    \section{Results and Inferences}\label{sec:Results}
  Experiment one compares the optimal algorithm (Alg.~\ref{alg:OptimalCollectWithOverlap}) versus three varients of the greedy algorithm (Alg.~\ref{alg:collectab}). 
  Fig.~\ref{fig:OptimalVsGreedy.pdf} compares the number of moves required to converge for the optimal strategy (Alg.~\ref{alg:OptimalCollectWithOverlap}) versus the greedy strategy for small worlds ranging from 5 free spaces to 30 free spaces.
  Workspaces with more spaces were not considered because the optimal algorithm requires a lot of time due to the exponential time to free space relationship.
   This plot shows that the optimal algorithm requires approximately half the moves of the greedy algorithms. The plot has an upward moving trend in general as the number of free spaces increases, but there are local minimums corresponding to easier configurations, which leads to downward spikes in the plot. The number of moves taken to completely collect particles also depends on the complexity of the workspace and does not completely depend on the number of free spaces. 

 For small workspaces  the best result among the greedy algorithms changes and so we cannot determine which is the best using small workspaces.
  To further compare the greedy strategies, we tested the algorithms on larger workspaces. The largest workspace in Fig.~\ref{fig:CollectingLeaf}  has 8,493 non-obstacle positions. Fig.~\ref{fig:CollectingLeaf} demonstrates that choosing which particles to pairwise collect in Alg.~\ref{alg:collectab} has a large impact on convergence time.
  We conducted a comparison study between the number of moves and the resulting unique particles. As discussed earlier in Alg.~1 (Section~\ref{Alg1}) and Alg.~2 (Section~\ref{Alg2}), getting the number of unique particles down to `1' signifies completion of the collecting algorithm. In the leaf vascular network, the majority of particle collection occurs during the first steps, with a long tail distribution to collect the final particles, as shown in the top row of images. Fig.~\ref{fig:CollectingLeaf}  shows that \emph{connect to first},  discussed in Section \ref{sec:theory}.B, outperforms the other algorithms. This can also be validated by further testing to compare the three greedy algorithms on larger workspaces.
 
  We simulated bounded worlds of varied sizes from 500 free spaces to 8,493 free spaces.  The results are represented in Fig.~\ref{fig:ConvergenceGreedyUsingLeafsOfVaryingSizes}, based on Fig.~\ref{fig:vascularNetwork} because biological vasculatures are our goal application. This graph plot has a smoother trend compared to the plot in Fig.~\ref{fig:OptimalVsGreedy.pdf}  because the ratio of free space to node complexity is similar for worlds in Fig.~\ref{fig:ConvergenceGreedyUsingLeafsOfVaryingSizes}. The important observation is the consistency that \emph{connect to first} performs best. This validates \emph{connect to first} as the best of the compared algorithms. This is good news because unlike the other two techniques, which involve distance calculation between all pairs of particles, there is negligible calculation involved in the \emph{connect to first}  algorithm. The data for which two particle are in top-leftmost location is readily available from the row, column indices of the particles.

   \begin{figure}
  \begin{overpic}[width=\columnwidth]{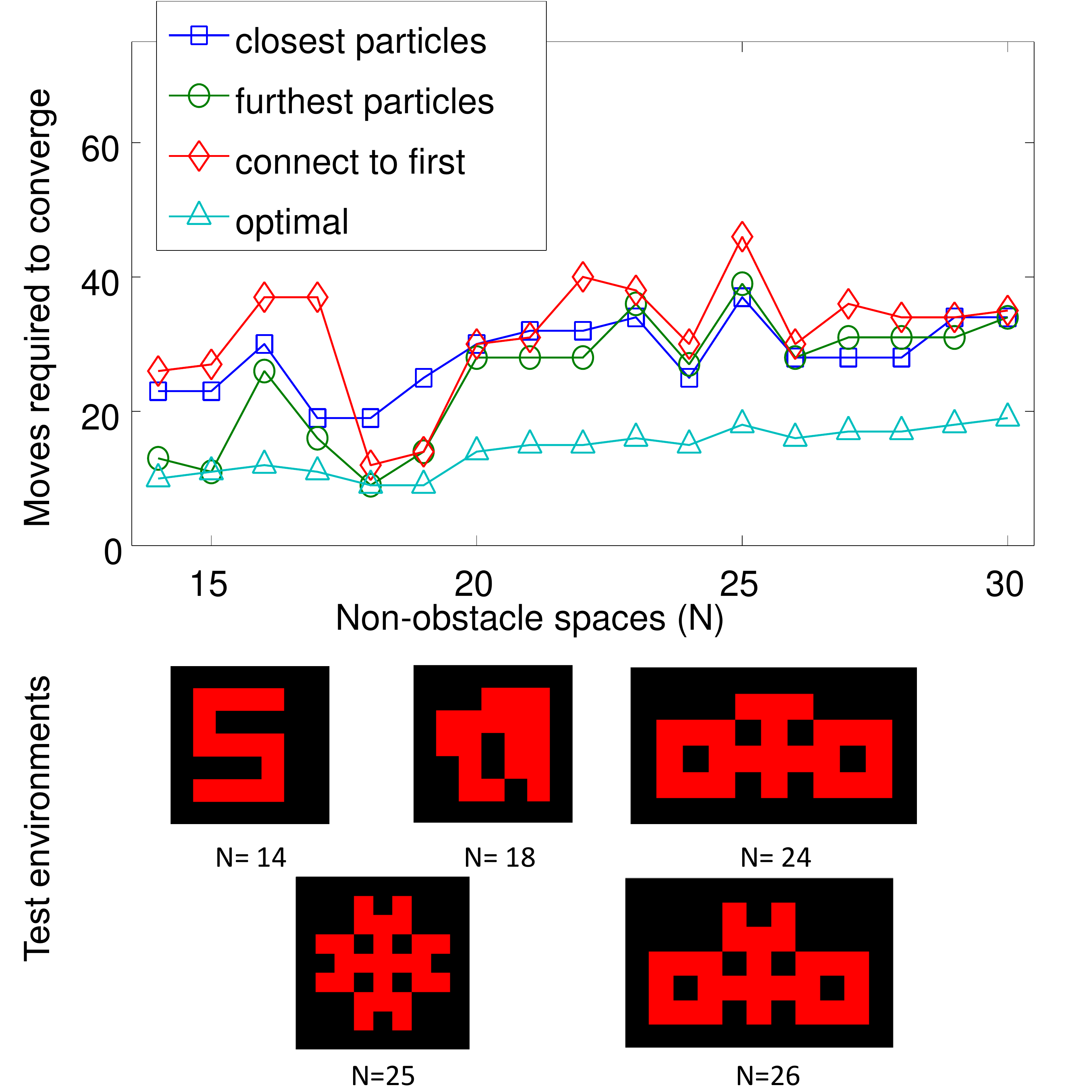}\end{overpic}
  \caption{\label{fig:OptimalVsGreedy.pdf} The greedy strategy requires 1.95 as many moves as the optimal strategy in a test with 17 different test environments. Below the plot are examples of some of the test environments used.  
   }
  \end{figure}

 \begin{figure}
  \begin{overpic}[width=\columnwidth]{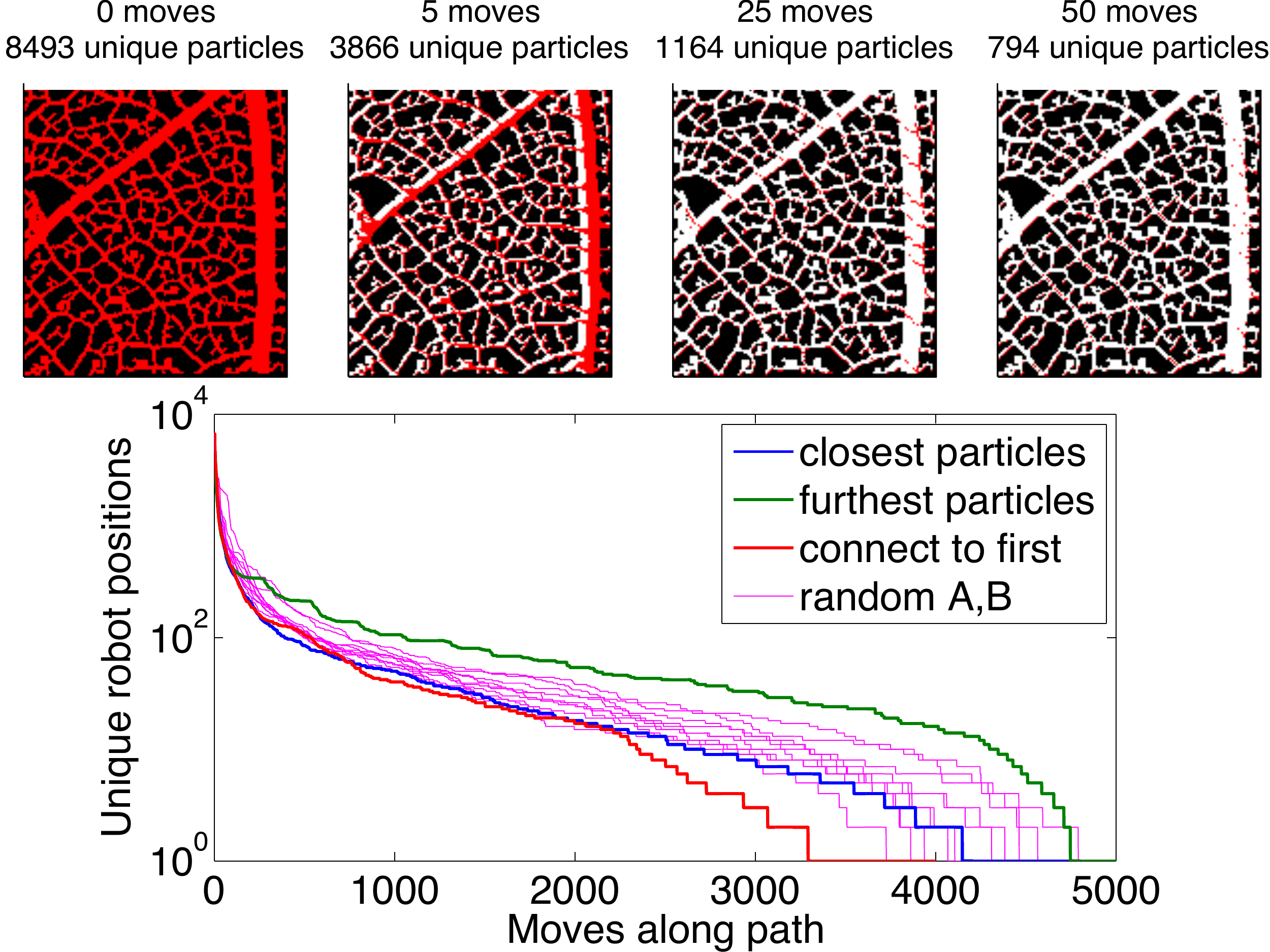}\end{overpic}
  \caption{\label{fig:CollectingLeaf}  Choosing which two particles to collect during each  {\sc CollectAB} step changes convergence time.    Collecting the furthest particles (green) performs poorly.  Connecting the two closest nodes (blue) is better, but both strategies are beat by the strategy that chooses the first two (the top-most, left-most) particles  each iteration.
 }
  \end{figure}

   \begin{figure}
  \begin{overpic}[width=\columnwidth]{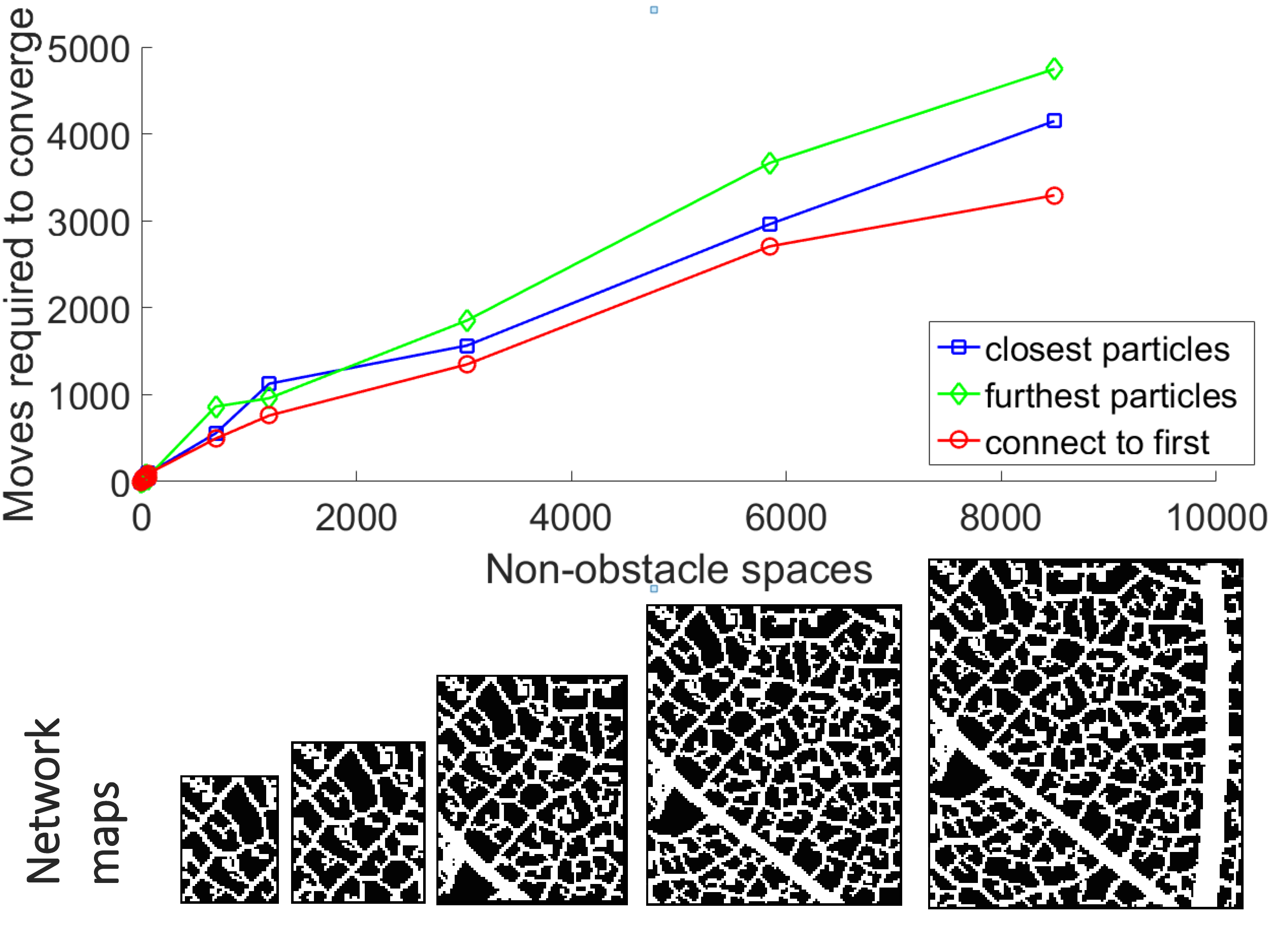}\end{overpic}
  \caption{\label{fig:ConvergenceGreedyUsingLeafsOfVaryingSizes} Comparing the required number of moves to collect overlapping particles to a position for three greedy algorithms (Alg.~\ref{alg:collectab}) for five different,  connected vascular networks.   }
  \end{figure}
  
     \section{Conclusion and Future work}\label{sec:Conclution} 
    This paper presented optimal and greedy algorithms to collect small globally commanded particles with guarantees that these algorithms will always collect particles for any bounded world which can be represented as a connected polyomino.
Algorithm \emph{connect to first} combines both low computational time and, in simulations, requires fewer moves than five other algorithms.
It  requires 50,607 moves to converge all the particles in the complete leaf world shown in Fig.~\ref{fig:vascularNetwork} (see video \cite{youtube}). 
   We also introduced challenges inherent with large particle collection, which poses new problems and complexities. 
The technology to fabricate microbots is rapidly improving and so has  interest in microrobots for potential applications in drug delivery. 
   There are many opportunities for future work, including refining the algorithms to handle large particles. 
   This paper assumed the workspace was bounded. That assumption is violated in biological vascular systems, which connect to larger vasculature.
    One avenue for future research is to add constraints to serve as virtual walls and actively prevent particles from escaping through a set of exits.  
   Additional complexities such as medium viscosity and wall friction must be studied before the algorithms are applied in vivo/in vitro. 
   Future work should focus not only on collecting, but also on avoiding accumulation in sensitive regions.

\section{Acknowledgements}
This material is based upon work supported by the National Science Foundation under Grant No.\ 
\href{http://nsf.gov/awardsearch/showAward?AWD_ID=1553063}{ IIS-1553063}.

   
\bibliographystyle{IEEEtran}
\balance
\bibliography{IEEEabrv,./bib/aaronrefs}%
\end{document}